\documentclass[letterpaper]{article}
\usepackage{aaai}
\usepackage{times}
\usepackage{helvet}
\usepackage{courier}
\usepackage{epsfig}
\usepackage{graphicx}
\usepackage{color}
\usepackage{amsmath}
\usepackage{amsthm}
\usepackage{subfigure}
\usepackage{multirow}
\usepackage{algorithm2e}

\usepackage[utf8]{inputenc} 
\usepackage{url}            
\usepackage{booktabs}       
\usepackage{amsfonts}       
\usepackage{nicefrac}       
\usepackage{microtype}      

\title{Unbiased Sparse Subspace Clustering By Selective Pursuit}
\author{Hanno Ackermann \and Michael Yang \and Bodo Rosenhahn}
% ...
\newcommand{\bm}{\mathbf}

\newcommand{\Bnu}{\boldsymbol{\nu}}
\newcommand{\Bc}{\bm{c}}

\newcommand{\Bs}{\bm{s}}

\newcommand{\Bv}{\bm{v}}

\newcommand{\Bx}{\bm{x}}
\newcommand{\By}{\bm{y}}

\newcommand{\CalH}{\mathcal{H}}
\newcommand{\CalI}{\mathcal{I}}

\newcommand{\CalS}{\mathcal{S}}

\newcommand{\RR}{I\!\!R} 
\newtheorem{corollary}{Corollary}

\newtheorem{proposition}{Proposition}

\DeclareMathOperator*{\argmax}{arg\,max}

\begin{document}

\maketitle

\begin{abstract}
Sparse subspace clustering (SSC) is an elegant approach for 
unsupervised segmentation if the data points of each cluster 
are located in linear subspaces. This model applies, for instance, 
in motion segmentation if some restrictions on the camera model 
hold. SSC requires that problems based on the $l_1$-norm are 
solved to infer which points belong to the same subspace. If 
these unknown subspaces are well-separated 
this algorithm is guaranteed to succeed. \newline
The algorithm rests upon the assumption that points on 
the same subspace are well spread. The question what happens if 
this condition is violated has not yet been investigated. 
In this work, the effect of particular distributions on the same 
subspace will be analyzed. It will be shown that SSC fails to infer 
correct labels if points on the same subspace fall into more than 
one cluster.

\end{abstract}


\section{Introduction}

This paper considers unsupervised classification of data of  
points which reside on multiple unknown and low-dimensional subspaces. 
The problem is to decide which points belong to the same subspace. 
This \emph{subspace clustering} arises in problems such as motion
segmentation~\cite{Kanatani01:MoSeg,Vidal2005:GCPR,Elhamifar13:SSC}, 
hand written digit clustering~\cite{Zhang2012:HLM}, 
face clustering~\cite{Kriegman2003:Faces} and
compression~\cite{Wright2005:Compression}. 

\emph{Sparse subspace clustering} (SSC) \cite{Elhamifar13:SSC} estimates 
the self-expressiveness within the data: which points can be 
used to linearly approximate a point in question? This is performed 
by minimizing $l_1$-norm problems. The support of each point 
is used to define the edge weights of a graph whose vertices correspond 
to the data points. Graph segmentation then reveals the class membership. 
In \cite{Candes2012:SSC}, SSC was theoretically analyzed and bounds for 
successful segmentation were derived.

Similar works use the nuclear norm instead of the $l_1$-norm to infer edge 
weights~\cite{Liu10:LRR}. An extension of SSC jointly estimates the 
parameters of a global subspace which includes all the 
data~\cite{Patel13:LatentSSC}. In a recent work, both steps of 
sparse optimization and spectral clustering were combined into a single, 
iterative algorithm~\cite{Vidal2015:S3C}.

A problem stems from the sparsity of the affinity matrix of the graph. As 
the edge weights are defined by the sparse coefficients, only few weights are known. 
The question then is whether the vertices of points on the same subspace
are always connected. 
It was first raised in 
\cite{Hartley2011:GraphConnectivity}. There, it was concluded that connectivity 
is guaranteed in subspaces of dimension $2$ and $3$, yet not for higher 
dimensions. For subspaces of dimension~$4$ an example was given in which 
the graph was disconnected.

In this paper we focus on the same question. Different prior works, we 
analyze the case that points on the same subspace are not evenly 
distributed as required in \cite{Candes2012:SSC}, but fall into different 
clusters. It will be shown that a few edges indeed connect vertices of 
different clusters. However, it will also be shown that the weights of such
edges are negligible small compared with edges connecting vertices of the
same cluster. In other words, the \emph{relative} connectivity is too low to
be significant. The reason is a bias of $l_1$-norm based estimators. 

The graph constructed from the sparse coefficients then consists of 
more disconnected subgraphs than the correct yet unknown number of 
subspaces. Since the subspace model is often motivated by some underlying 
physical model, for instance each rigidly moving body induces a
$4$-dimensional subspace in motion segmentation, over-segmenting 
the data can be difficult to correct. 
On the other hand, because most algorithms of this 
class~\cite{Elhamifar13:SSC,Patel13:LatentSSC,Vidal2015:S3C,Liu10:LRR} 
require the number of clusters to be known in advance, cluster centers 
will be incorrect which causes the labelling to be wrong. It is therefore 
more desirable to address the original problem in the first place.

This paper is structured as follows: In Sec.~\ref{Sec:Notation}, we  
summarize the notation used in this paper. In Sec.~\ref{Sec:SSC}, the original 
sparse subspace clustering algorithm is shortly explained. An example of the
effect of non-uniformly distributed points is given in 
Sec.~\ref{Sec:Example}. A formal analysis on the magnitude of the 
sparse coefficients in given in Sec.~\ref{Sec:Connectivity}. 
The proposed algorithm is introduced 
in Sec.~\ref{Sec:Algo}. Experimental results are shown in 
Sec.~\ref{Sec:Exps}. The paper concludes with a summary in 
Sec.~\ref{Sec:Conclusions}.


\section{Notation}
\label{Sec:Notation}

\begin{figure*}[t]
  \includegraphics[width=0.22\textwidth]{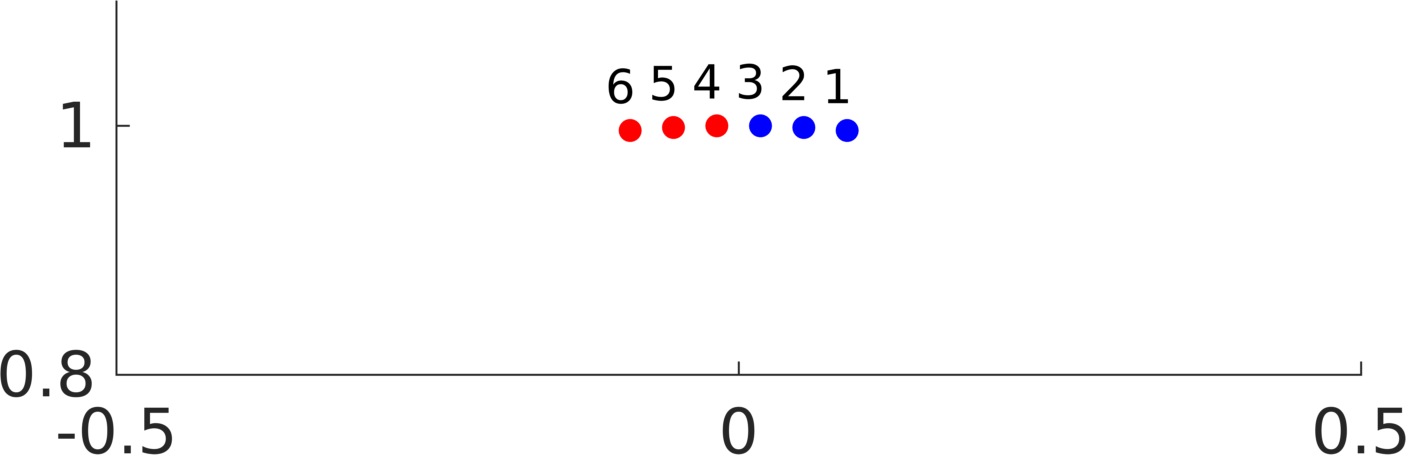}
  \includegraphics[width=0.22\textwidth]{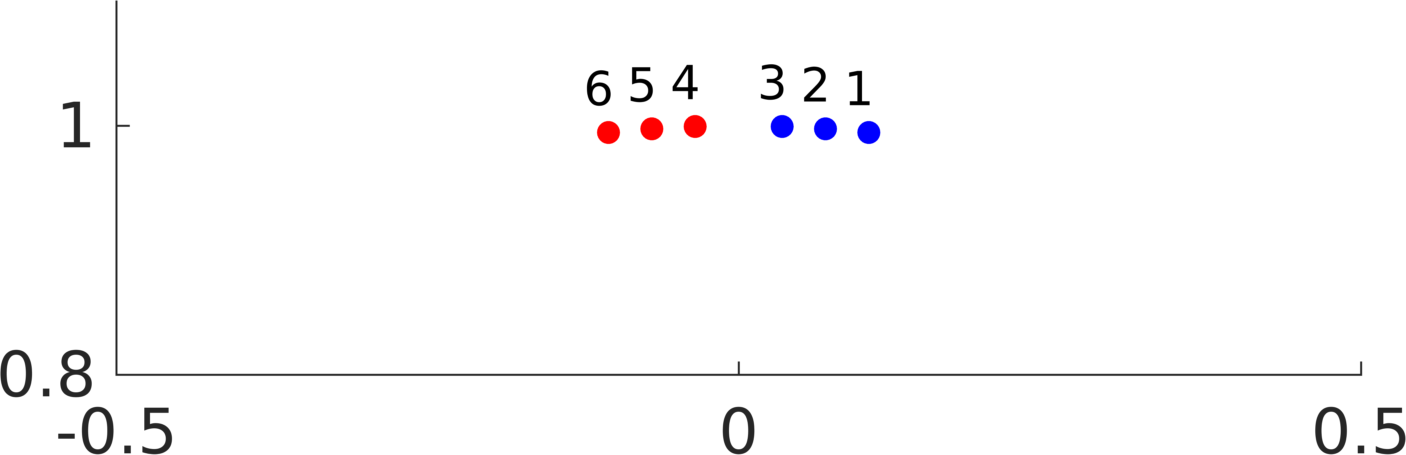}
  \includegraphics[width=0.22\textwidth]{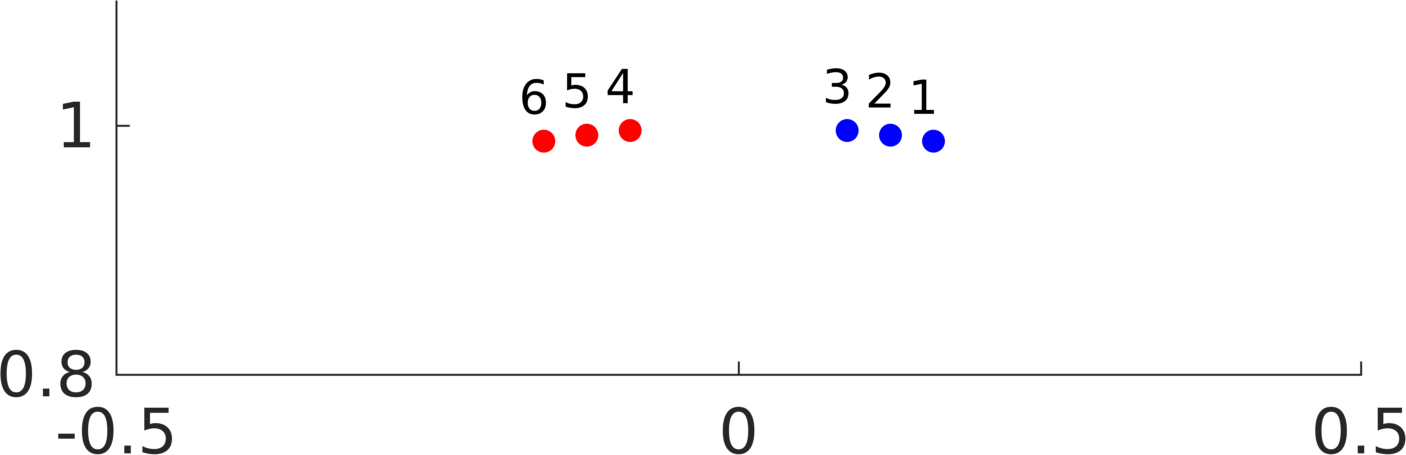}
  \includegraphics[width=0.22\textwidth]{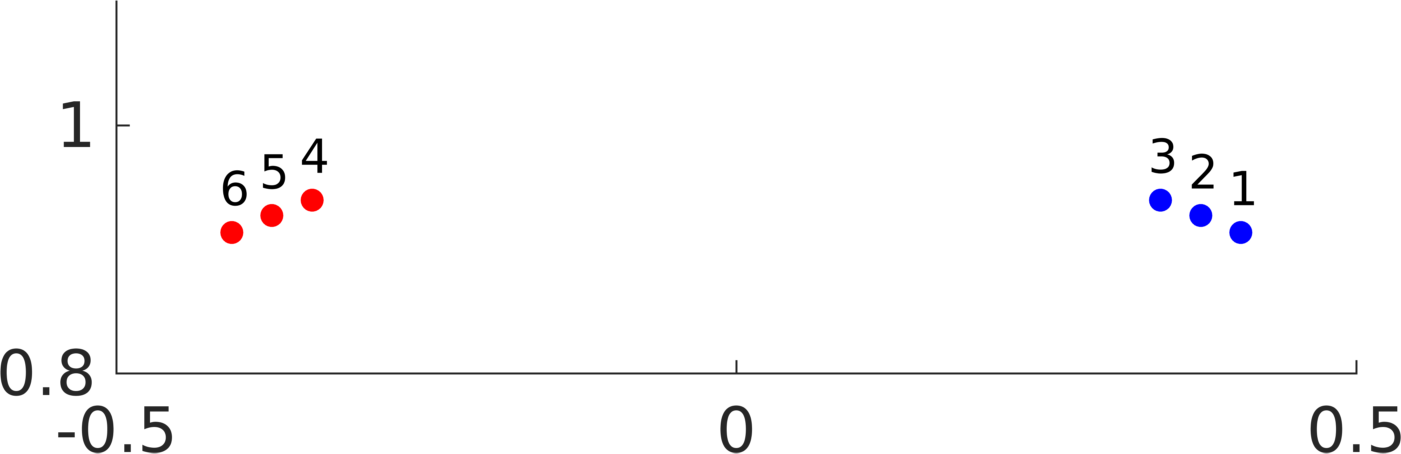}
  \includegraphics[width=0.22\textwidth]{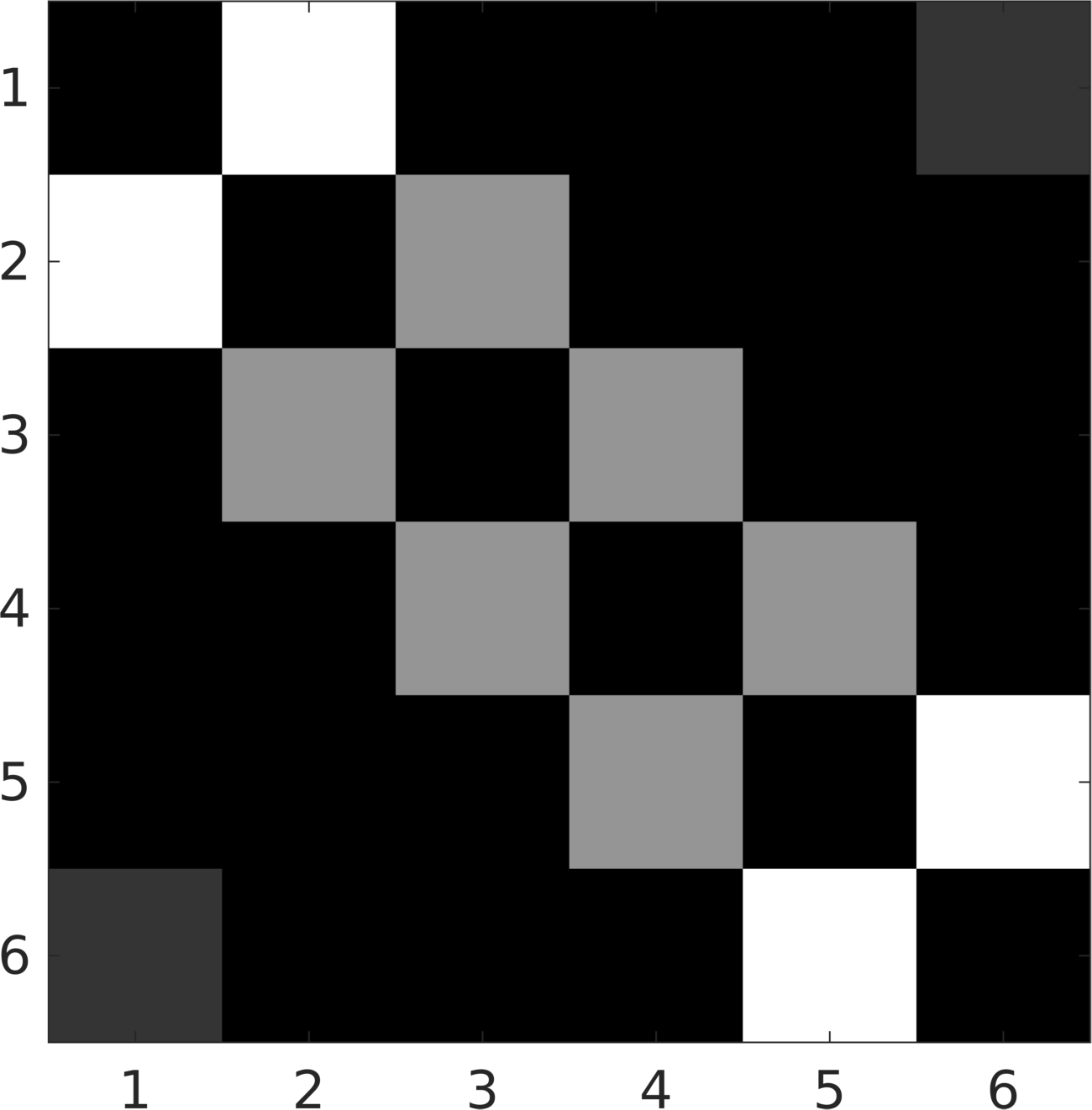}
\hspace{0.55cm}
  \includegraphics[width=0.22\textwidth]{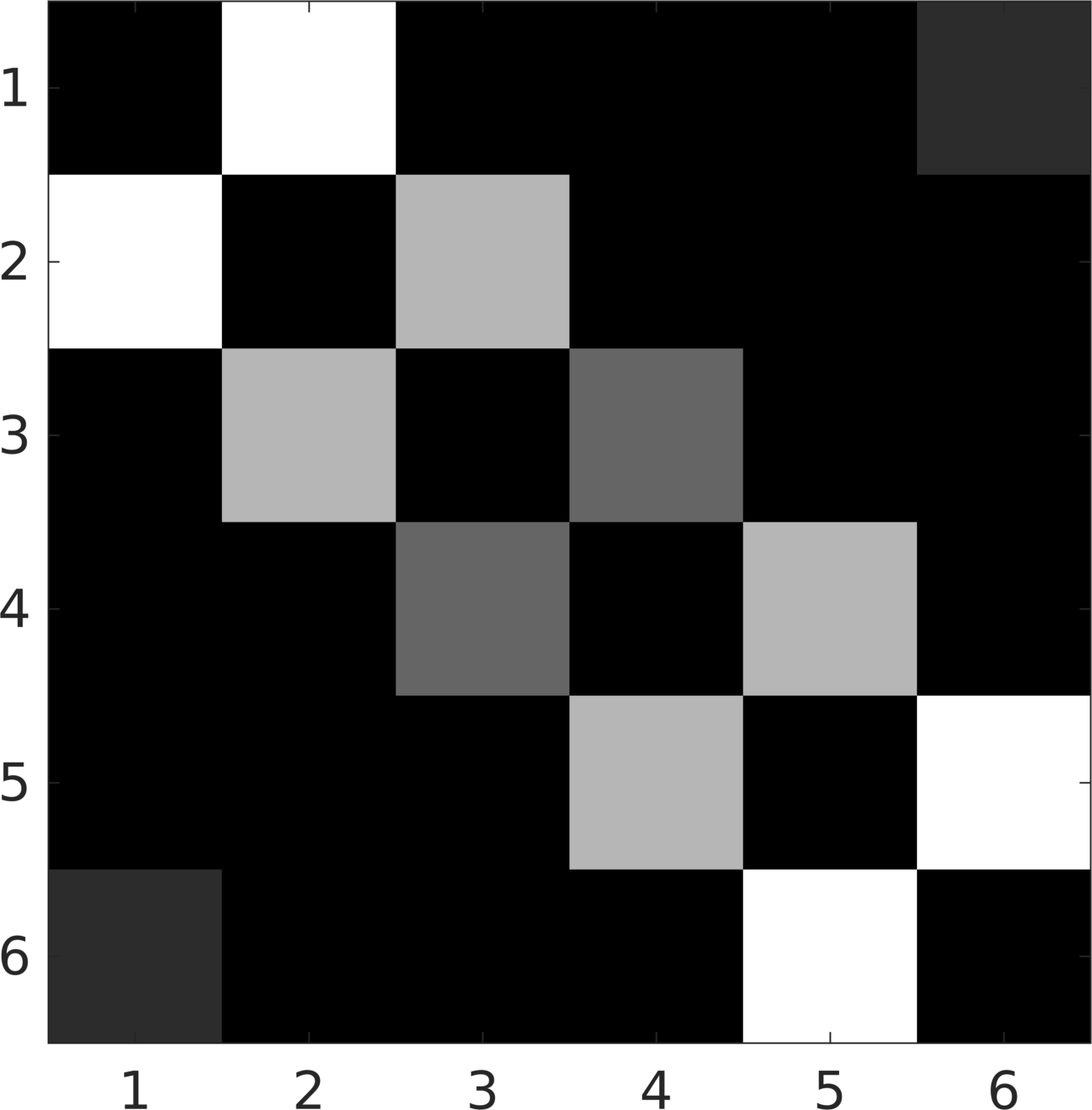}
\hspace{0.55cm}
  \includegraphics[width=0.22\textwidth]{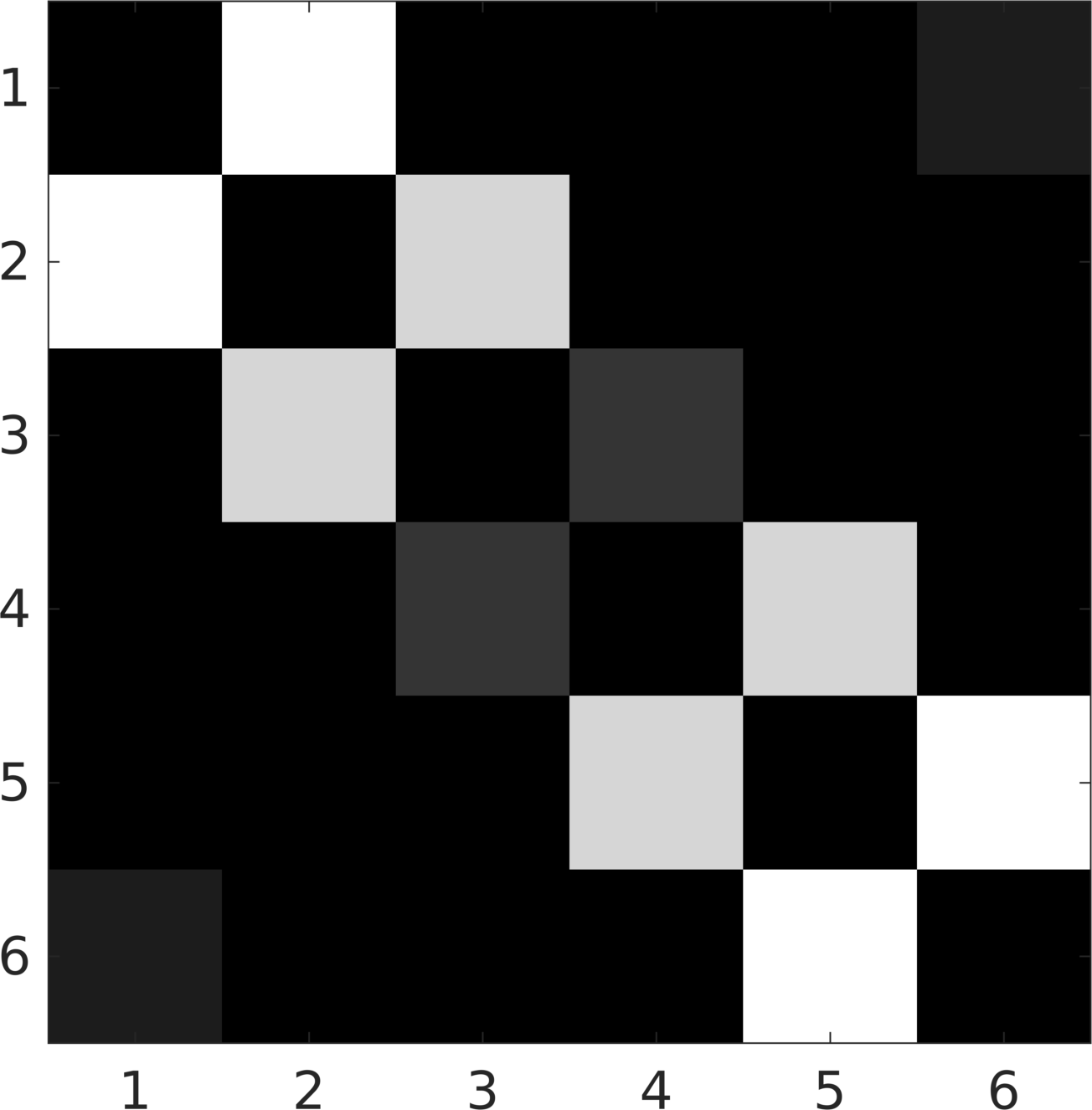}
\hspace{0.55cm}
  \includegraphics[width=0.22\textwidth]{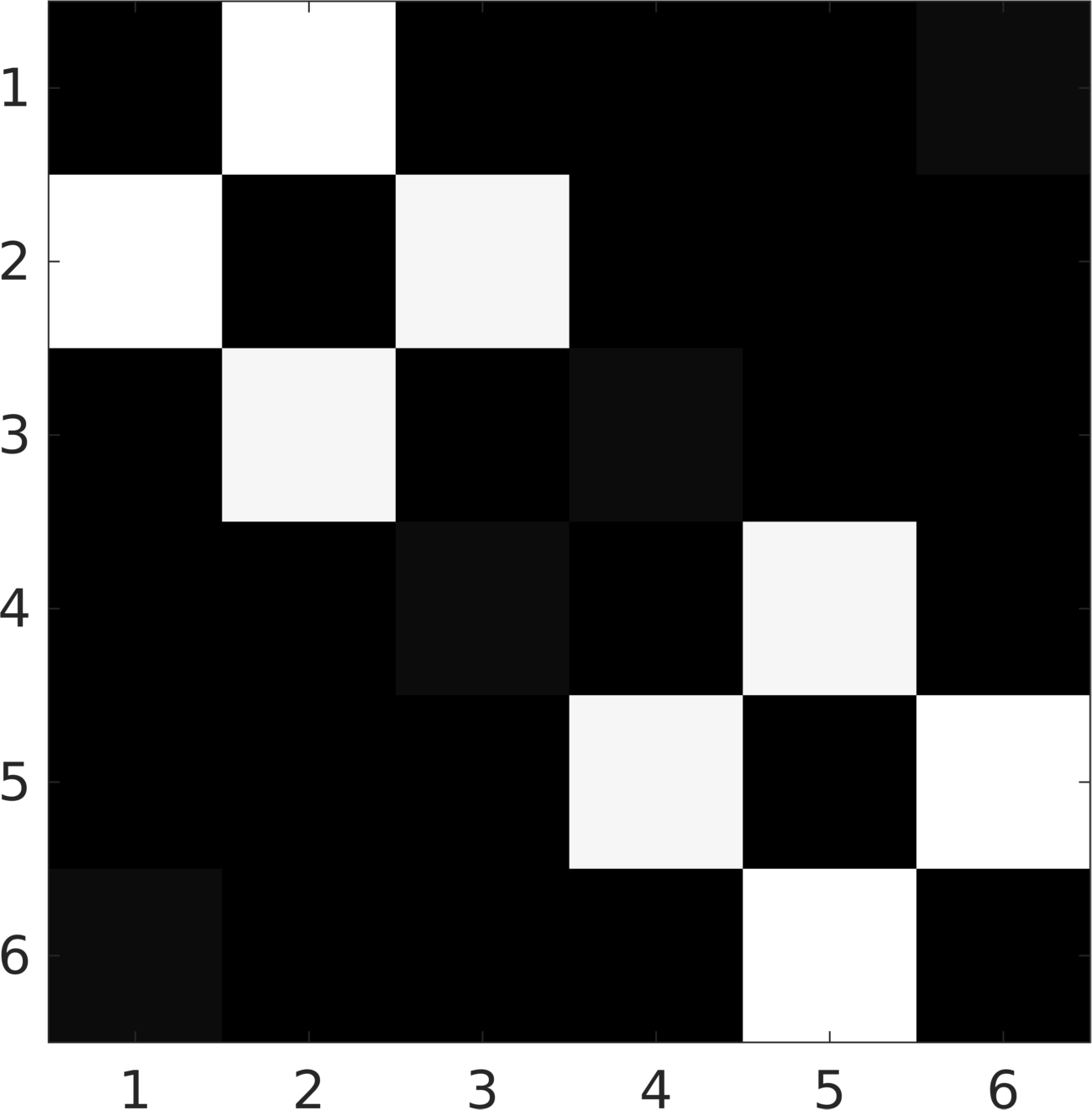}
  \caption{Upper row: two sets of points (red and blue) separated by angles of
  $2^\circ$, $4^\circ$, $10^\circ$ and $20^\circ$ from left to right. Lower
  row: corresponding connectivity matrices $|C|+|C^\top|$ resulting from
  solving Eq.~\eqref{Eq:L1Optim} with $\lambda = 0.01$. The columns/rows
  correspond with the numbering of the points above. \newline
  The larger the angle between the two sets of points, the lower the
  connectivity between the two clusters as the magnitude of the entries on the
  anti-diagonal becomes smaller and smaller. For an angle of $20^\circ$, the
  two sets almost form two disconnected clusters.} 
  \label{Fig:Connectivity.Example}
\end{figure*}

Bold lower-case letters $\Bx$ indicate vectors, normal capital letters 
matrices, e.g. $X = [ \Bx_1 \; \cdots \; \Bx_N ]$, or constants, e.g. $N$. 
By $\Bx(j)$, we mean the $j$th element of vector $\Bx$, and by $A(i,j)$ the
$(i,j)$th entry of matrix $A$. 
Sparse subspace clustering rests upon the assumption that there are $L$
subspaces $\CalS^{(l)} \in \RR^m$, $l=1,\ldots,L$ with $N_l$ points 
$X^{(l)} = [ \Bx^{(l)}_1,\ldots,\Bx^{(l)}_{N_l}]$ on it. Let $d^{(l)}$ denote 
the dimension of $\CalS^{(l)}$. Let 
$X = [ X^{(1)} \, \ldots \, X^{(L)} ]$ denote the matrix of all 
$N_1 + \cdots + N_L = N$ data points. Assume without loss of generality 
that $X$ is ordered. The points $\Bx_j$ are supposed to have unit length 
$\| \Bx_j \|_2 = 1$.

Let $T^{(l)}$ indicate the set of indices of all the points 
$\Bx^{(l)} \in \CalS^{(l)}$. Let $S_j \subseteq T^{(l)}$ indicate the support 
set of point $\Bx_j$, $j=\{1,\ldots,N\}$. 
The cardinality of a set is denoted by $| \cdot |$. 
The $k$th point of the support set $S_j$ of $\Bx^{(l)}_j$ is denoted by 
$\Bx^{(l)}_{j,k}$. 
Let $\By$ be any of the points of $X$. The matrix $X_{-y}$ indicates the 
matrix $X$ without the column corresponding to $\By$.

By $\| \cdot \|_{\{0,1,2,F\}}$ we indicate the $l0$-pseudo-norm, or the $l1$-, 
$l2$-, or Frobenius-norm of the argument.

\section{Sparse Subspace Clustering}
\label{Sec:SSC}


\emph{Sparse subspace clustering} (SSC)~\cite{Elhamifar13:SSC} is based 
on the fact that any point $\Bx_j \in \CalS^{(l)}$ can be expressed by a 
linear combination of $d^{(l)}$ other points in the same subspace. Since 
such linear combinations do not use points of other classes, they can 
be used to infer the unknown class labels.



The problem is to compute the least number of points $\Bx_k$, 
$k \neq j$ such that a linear combination of the $\Bx_k$ yields $\Bx_j$. 
Denote by $\Bc_j$ the vector of mixing coefficients with its $j$th entry 
being equal to zero such that $\Bx_j = X \Bc_j$ holds true. The task to 
estimate the vector $\Bc_j$ can be solved by minimizing
\begin{align}
  \min{ \|\Bc_j \|_0 } \quad  s.t. \quad \| X \Bc_j - \Bx_j \|_2^2 = 0 \;\; \mbox{and} \;\; c_{jj} = 0.   
  \label{Eq:L0Optim}
\end{align}

Since optimizing the $l_0$-norm is difficult, 
it is often approximated by using the $l_1$-norm. Allowing for 
data points $\Bx_j$ contaminated by noise, it is possible to instead
optimize
\begin{align}
  \min{ \|\Bc_j \|_1 } \quad  s.t. \quad \| X \Bc_j - \Bx_j \|_2^2 \leq \lambda \;\; \mbox{and} \;\; c_{jj} = 0.   
  \label{Eq:L1Optim}
\end{align}
for a scalar $\lambda>0$.

Given the matrix $C = \begin{bmatrix} \Bc_1 & \cdots & \Bc_N \end{bmatrix}$, 
class labels can be inferred by means of spectral clustering of the graph 
$G=(V,E)$ with $|V|=N$ vertices corresponding to the points $\Bx_j$ and edge 
weights $E_{ij}$ defined by the matrix $|C|+|C|^\top$. These steps are 
the basis of the so-called sparse subspace clustering (SSC) algorithm 
proposed in \cite{Elhamifar13:SSC}.



\section{Subspace Connectivity in $\RR^2$}
\label{Sec:Example}



Since only few edge weights are known, the authors of
\cite{Hartley2011:GraphConnectivity} raised the question whether connectivity 
between vertices of the same cluster is always guaranteed. In other words, 
is the subgraph consisting of the vertices of one particular cluster
connected? Their answer was that connectivity is guaranteed -- and thus a 
correct result of SSC -- if the subspace dimension is $d=2$ or $d=3$, 
but an example was given for $d=4$ where connectivity is
violated. Independently, the authors of \cite{Candes2012:SSC} concluded that 
SSC clusters correctly as long as the points are well spread within each 
cluster.

In this work, we will analyze the connectivity if the data is not well spread 
across each subspace but forms more or less well isolated clusters \emph{in a
  particular subspace}. As an example, consider
Fig.~\ref{Fig:Connectivity.Example}. Here, the upper row shows plots of two
sets of points (depicted in red and blue). From left to right, the angle
between the two sets increases from $2^\circ$ to $4^\circ$ and $10^\circ$
until $20^\circ$. The bottom row of Fig.~\ref{Fig:Connectivity.Example} 
shows the corresponding affinity matrices $|C|+|C^\top|$ resulting from
optimizing Eq.~\eqref{Eq:L1Optim} with $\lambda=0.01$. The columns and rows
correspond to the numbering of the points in the upper row of
Fig.~\ref{Fig:Connectivity.Example}. It can be seen that the $(3,4)$ and
$(4,3)$ entries have the same magnitudes as the $(2,3)$, $(3,2)$ and $(4,5)$,
$(5,4)$ entries for an angle of $2^\circ$. However, at an angle of $20^\circ$,
the entries on the anti-diagonal have almost negligible magnitude as compared
to the other entries on the diagonal. Therefore, the resulting graph consists
of two almost disconnected clusters.

Apparently, the question raised in \cite{Hartley2011:GraphConnectivity}, namely
whether there is connectivity, i.e. are there non-zero entries in the affinity
matrix, is not sufficient. The affinities shown on the anti-diagonal of the
rightmost affinity are $>0$ yet their magnitude is negligible as compared to
the other entries. 

This example shows that spurious clusters can emerge if sparse coefficients 
are used as affinities. The example is equivalent to points on an affine line
($d=2$) normalized to unit length. It therefore readily generalizes to 
higher-dimensional structures in $\RR^m$, $m>2$. The remainder of this paper
focuses on the question how gaps between data on the same subspace influence
the solutions $\Bc_j$ of Eq.~\eqref{Eq:L1Optim} and thus the \emph{relative}
connectivity. The relative connectivities then determine if the corresponding
vertices of the graph form more or less disconnected clusters.


\section{The Bias of $l_1$-Norm Estimators}
\label{Sec:Connectivity}

The analysis in this section solely concentrates on points on a single
subspace $\CalS^{(l)}$. For the sake of simplicity, the superscript $(l)$ 
is omitted in the following. 

Let the angle between two points $\Bx_{j_1}$ and $\Bx_{j_2}$, $j_1 \neq j_2$ be 
defined by the inverse cosine of the scalar product of two unit length
vectors, $\angle(\Bx_{j_1},\Bx_{j_2}) = \mbox{acos}(\Bx_{j_1}^\top \cdot \Bx_{j_2})$. 
Assume without loss of generality that the points $\Bx_{j,k}$ are ordered 
such that 
$\angle(\Bx_j,\Bx_{j,1}) \geq \cdots \geq \angle(\Bx_j,\Bx_{j,K})$. 


\subsection{Noiseless $l_1$-Norm: $\lambda = 0$}
\label{Sec:Sub:Noiseless}
Given a point $\By$, a matrix $X_{-y}$ and the $K$ points $\Bx_{y,k}$,
$k=1,\ldots,K$ of a support set $S_y$ of $\By$, 
and let $\Bc_y$ be a solution to $\| X_{-y} \Bc_y - y \|_2^2 \leq \lambda$ so that all 
entries but those corresponding to the points $\Bx_{y,k}$ in $X_{-y}$ 
are zero. The question then is how does $\| \Bc_y \|_1$ change if a 
different support set $S_y^\prime \subseteq T_y$ is selected.


The geometry of this case is shown in the left plot of 
Fig.~\ref{Fig:Geometry}. The intersection of the red dash-dotted line with 
the line through the point $\Bx_{y,1}$ indicates the point $c_{y}(1) \Bx_{y,1}$. 
The point $\Bs$ indicates the intersection of this line with the subspace 
 $\mbox{span} (Q)$ spanned by the remaining points $Q = \begin{bmatrix}
   \Bx_{y,2} & \cdots & \end{bmatrix}$. It can be seen that $\|\Bs\|_1$ grows
 as the angle $\theta$ between $\mbox{span} (Q)$ and $\mbox{span} (\Bx-\By)$
 decreases. This intuition motivates the following proposition:

\begin{figure*}[t]
    \mbox{\includegraphics[width=0.3\textwidth]{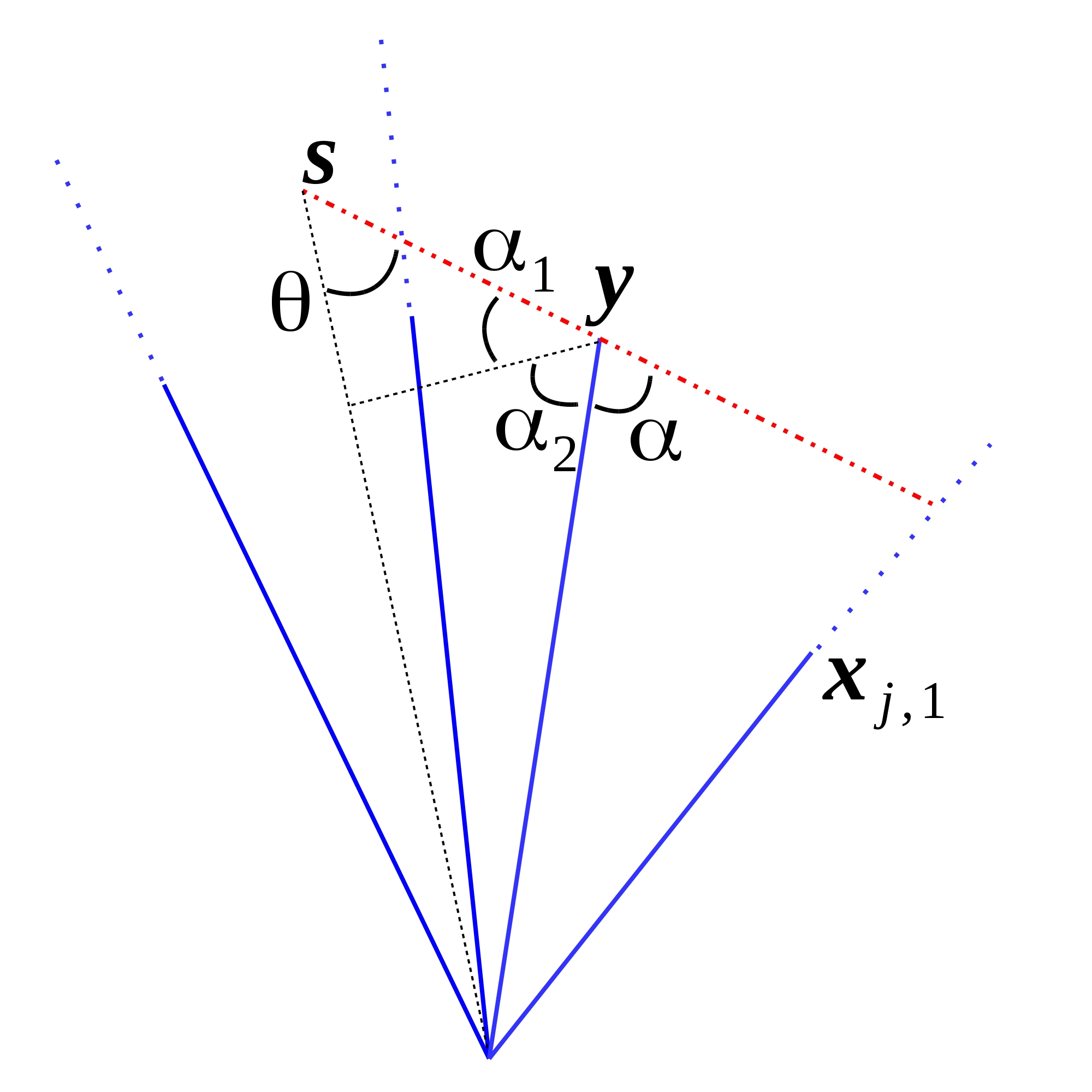}}
    \mbox{\includegraphics[width=0.3\textwidth]{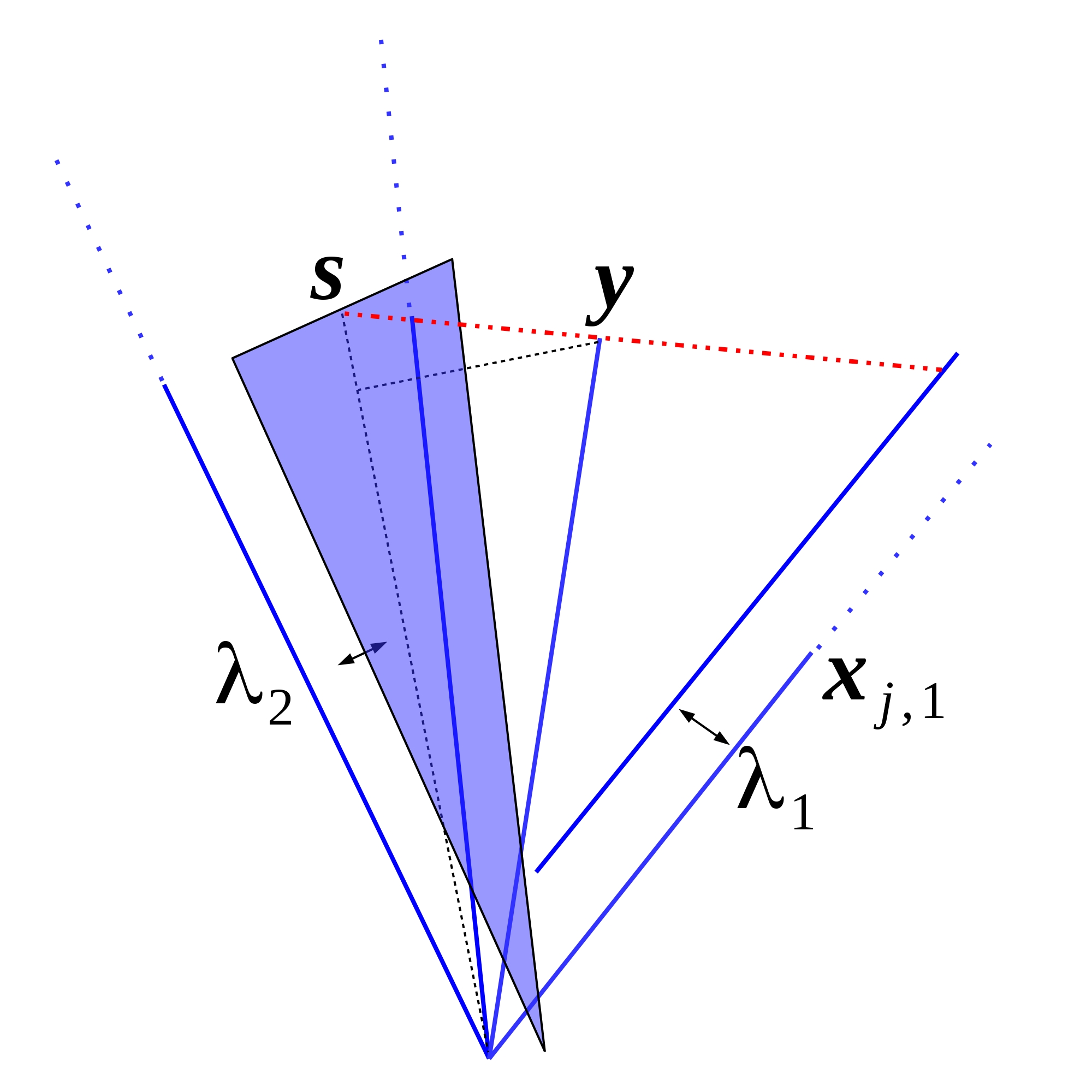}}
    \mbox{\includegraphics[width=0.3\textwidth]{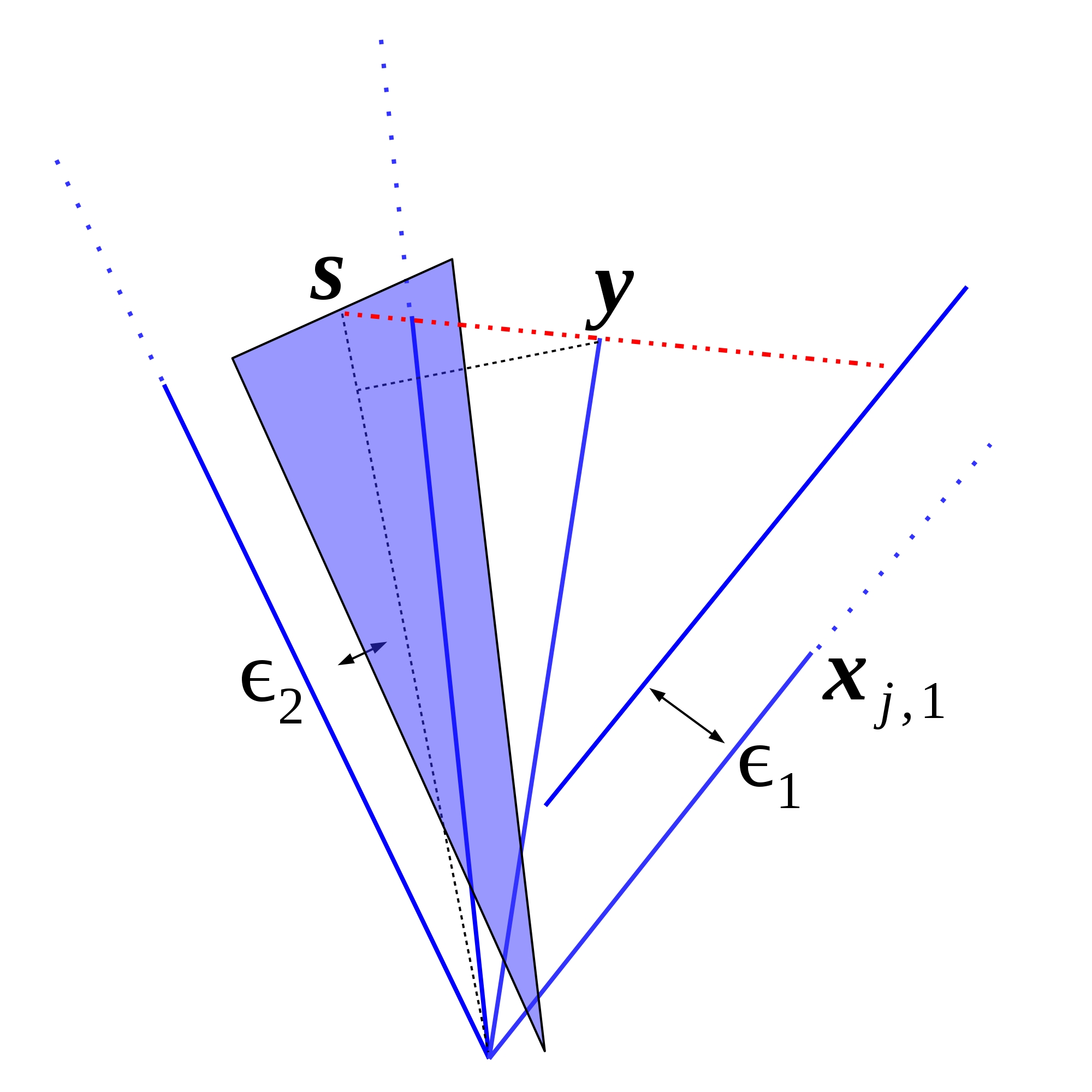}}
  \caption{Geometry of different $l_1$-norm estimators. Left: noiseless case 
    (Eq.~\eqref{Eq:L1Optim} with $\lambda = 0$); middle: noisy $l_1$-norm 
    (Eq.~\eqref{Eq:L1Optim} with $\lambda > 0$); right: lasso.}
\label{Fig:Geometry}
\end{figure*}

\begin{proposition}
The $l_1$-norm of $\Bc_y$ increases with the angle 
$\angle \left( \By, \Bx_{y,1} \right)$ between $\By$ and $\Bx_{y,1}$
  \begin{equation}
    \| \Bc_j \|_1 \geq \cot \left(f \left( \angle \left( \By, \Bx_{y,1} \right) \right)  + \rho - \frac{\pi}{2} \right) \cdot \eta.
    \label{Eq:Noiseless.Limit.l1}
  \end{equation}
  \label{Prop:Noiseless.Dependency}
\end{proposition}
\begin{proof}
Let $\Bv_1 = (c_{j,1}\Bx_{y,1}-\By) / \| c_{j,1}\Bx_{y,1}-\By \|_2 $,  
$\Bv_2 = (P_Q \By - \By) / \| P_Q \By - \By \|_2$ where $P_Q$ is the 
orthogonal projector onto $\mbox{span}(Q)$, and $f\left( \angle \left( \By,
    \Bx_{y,1} \right) \right) = \cos^{-1} \Bv_1$. 
Let $\alpha = \cos^{-1} \Bv_1^\top \By$ be the angle between the line connecting 
$\By$ and $c_{j,1}\Bx_{y,1}$. Assuming that $|T_{y}| > |S_{y}|$, the 
support set $\Bx_{y,2},\ldots$ is selected such that $\mbox{span}(Q) 
\perp \mbox{span}(\begin{bmatrix} \By & \Bx_{y,1} \end{bmatrix})$ 
in order to minimize $\|\Bx\|_2$. 
Then, we can define the angles $\alpha_1 = \cos^{-1} -\By^\top \Bv_2$, 
$\alpha_2 = \pi - \alpha - \alpha_1$, and $\theta = \alpha + \alpha_1 
- \pi/2$ which motivates
\begin{equation}
  \| \Bx - P_Q \By \|_2 = \mbox{cotan} \, \theta \cdot \| (I-P_Q) \By \|_2.
\end{equation}
If we now define $\Bnu = (I-P_Q)\By$, $\eta = \| \Bnu \|_2$, and 
$\rho = \cos^{-1} (\By^\top \Bnu) / \eta$ we arrive at the claim since 
$\|\By\|_1 \geq \| \By \|_2$.
\end{proof}

Given $\By$, and $d$ points $\Bx_{y,1},\ldots,\Bx_{y,d}$ of a support set
$S_y$ and the corresponding sparse solution $\Bc_S$ to 
$\|X_{-y}\Bc_S-\By\|_2=0$. Let $\CalH (\By,Q)$ be the plane through 
$\By$ with orthonormal basis $Q$. Assume that there is a point 
$\Bx_n \in X_{T_y \setminus S_y}$ with 
$\angle \left( \By, \Bx_n \right) < \angle \left( \By, \Bx_{y,1} \right)$ 
on the same side of $\CalH (\By,Q)$ as $\Bx_{y,1}$.  Let $S^\prime_y$ be the 
set of indices of the points $\Bx_n,\Bx_{y,2},\ldots,\Bx_{y,d}$ and 
$\Bc_{s^\prime}$ be the solution to $\|X_{-y}\Bc_{S^\prime}-\By\|_2=0$. 
\begin{proposition}
  The solution $\Bc_{s^\prime}$ of $\|X_{-y}\Bc_{S^\prime}-\By\|_2=0$ satisfies
  \begin{equation}
    \|\Bc_{S^\prime}\|_1 < \|\Bc_{S}\|_1.
  \end{equation}
  \label{Prop:Swap}
\end{proposition}
\begin{proof}
 The proof follows from Prop.~\ref{Prop:Noiseless.Dependency}.
\end{proof}

\begin{corollary}
The $d$ points $\Bx_{y,1},\ldots,\Bx_{y,d} \in X_{T_y}$ closest to $\By$ minimize $\|\Bc\|_1$ such that 
$\|X_{-y}\Bc-\By\|_2=0$.
\end{corollary}



The implication 
is as 
follows: The usual assumption is that the data uniformly distributed. 
In the context of sparse subspace clustering, this implies that data 
on the same subspace need be uniformly distributed 
(cf.~\cite{Candes2012:SSC}). If this condition is violated, for instance
because points fall into two well separated clusters, the support of the
points of one cluster will not include points of the other.

This does not change the results in~\cite{Candes2012:SSC} since the
central Theorem~2.5 there rests upon
the assumption that points on the same subspace are more or less evenly
distributed. The idea in this work considers a particular violation of 
this assumption.

\subsection{Robust $l_1$-Norm: $\lambda > 0$}
Define $\lambda_1$ and $\lambda_2$ such that $\lambda_1 + \lambda_2 =
\lambda$. The idea here is to define an auxiliary line parallel to
$\mbox{span} \begin{bmatrix} \Bx_{y,1} \end{bmatrix}$ with distance
$\lambda_1$ into direction of $\By$, and an auxiliary subspace parallel to
$\mbox{span} \begin{bmatrix} \Bx_{y,2} & \cdots & \end{bmatrix}$ with distance
$\lambda_2$ also into direction of $\By$. The geometry of this configuration
is shown in the middle of Fig.~\ref{Fig:Geometry}. 

\begin{proposition}
The $l_1$-norm of $\Bc$ increases with the angle 
$\angle \left( \By, g\left( \Bx_{y,1},\lambda_1 \right) \right)$ between $\By$
and $g\left( \Bx_{y,1}, \lambda_1 \right)$
  \begin{multline}
    \| \Bc \|_1 \geq h_2\left( \eta, \lambda_2 \right) \cdot \\ 
    \cot \left(f \left( \angle \left( \By, g\left( \Bx_{y,1}, \lambda_1 \right) \right) \right)  + h_1\left( \rho,\lambda_2 \right) - \frac{\pi}{2} \right).
    \label{Eq:Noisy.Limit.l1}
  \end{multline}
  \label{Prop:Noisy.Dependency}
\end{proposition}
\begin{proof}
The idea for the robust $l_1$-norm estimator is to create auxiliary variables 
$\lambda_1$ and $\lambda_2$ so that $\lambda = \lambda_1 + \lambda_2$. These two 
are used to define two affine spaces parallel to $\mbox{span}(\Bx_{y,1})$ and 
$\mbox{span}(Q)$ with distances $\lambda_1$ and $\lambda_2$ such that the 
distances to $\By$ are reduced by $\lambda_1$ and $\lambda_2$. 
The function $g$ constructs the affine space parallel to $\mbox{span}(\Bx_{y,1})$. 
The functions $h_1$ and $h_2$ modify the projection of $\By$ not onto 
$\mbox{span}(Q)$ but its parallel affine space, and correct $\rho$, respectively. 
\end{proof}

\subsection{Lasso}
The configuration for the lasso estimator
\begin{equation}
  \min \;\| \Bc \|_1 + \lambda \cdot \| X_{-y} \Bc - \By \|_2
\label{Eq:Lasso}
\end{equation}
is very similar to the one shown for the robust $l_1$-norm and is shown 
in the right plot of Fig.~\ref{Fig:Geometry}.   
We can now define a variable $\epsilon = \epsilon_1 + \epsilon_2$ that it 
absorbs the $l_2$-error. 
Here $\epsilon_1$ 
indicates the distance of an auxiliary line parallel to
$\mbox{span} \begin{bmatrix} \By \end{bmatrix}$ with distance $\epsilon_1$ and
an auxiliary subspace parallel to $\mbox{span} \begin{bmatrix} \Bx_{y,2} &
  \cdots & \end{bmatrix}$ with distance $\epsilon_2$. The difference to the
robust $l_1$-norm is that $\epsilon$ is now variable.

Due to the close similarity between the two estimators, 
Proposition~\ref{Prop:Noisy.Dependency} can easily be adapted to the lasso estimator:
\begin{proposition}
The $l_1$-norm of $\Bc$ is increases with the angle 
$\angle \left( \By, g\left( \Bx_{y,1},\lambda_1 \right) \right)$ between $\By$
and $g\left( \Bx_{y,1}, \lambda_1 \right)$
  \begin{multline}
    \| \Bc \|_1 \geq h_2\left( \eta, \epsilon_2 \right) \cdot \\
    \cot \left(f \left( \angle \left( \By, g\left( \Bx_{y,1}, \epsilon_1 \right) \right) \right)  + h_1\left( \rho,\epsilon_2 \right) - \frac{\pi}{2} \right).
    \label{Eq:Lasso.Limit.l1}
  \end{multline}
  \label{Prop:Lasso.Dependency}
\end{proposition}
\begin{proof}
The proof of Proposition~\ref{Prop:Lasso.Dependency} is identical as the one of 
Prop.~\ref{Prop:Noisy.Dependency} except that the admissible error 
$\epsilon = \epsilon_1 + \epsilon_2$ is now variable.
\end{proof}





\section{Selective Pursuit}
\label{Sec:Algo}









\subsection{Selective Dantzig Selector}

Inspired by the Dantzig selector~\cite{Candes2007:Dantzig,Qu2015:RobustDantzig}, 
$X^\star = X^\top$, we notice that the product $X^\star X$ measures 
correlations to the subspace if we define a modified Dantzig selector 
\begin{equation}
  X^\star = X_{S_{y}}^\top.
\end{equation}
The matrix $X_{S_{y}}$ consists of those columns of $X$ that are the 
support of $\By$. Since points on the same subspace have smaller 
distances to it than points from different subspaces, we can 
expect that the products $X^\star \Bx_j$, $j = 1,\ldots,N \; \setminus \,
\{\CalI(\By),\, \mbox{supp}(\By)\}$ to be large if $j \in T_{y}$. 

The idea proposed here is as follows: The support set $S_{y}$ of $\By$
usually consists of points from the set of neighbors of $\By$. Rather than 
stopping at this point, it is possible to select additional points,
if their coherences with $\mbox{span}(X_{S_{y}})$ is are large. 

Let $S^e_{y}$ be the extended support set of $\By$ which consists of 
the original support set, and the additionally selected points. If 
the modified Dantzig selector is now taken to be $X^\star = X_{S^e_{y}}^\top$, 
the selection process becomes more and more influenced by noise. Because 
noise causes spurious singular values in $X^\star$, we reduce their effect 
by scaling with 
\begin{equation}
  \rho = \mbox{trace}(X^\star)^\top X^\star.
\end{equation} 
Since the trace of the square of a  matrix equals the sum of its singular 
values, this reduces the effect of the spurious singular values. Finally, 
new extended support vectors are chosen if for a threshold $\delta$
\begin{equation}
  \argmax \limits_{j} \frac{\| X^\star \Bx_j \|_2^2 }{\rho} > \delta, \quad j \in \{1,\ldots N\}
  \setminus \{\CalI(\By),\, S^e_{y} \}.
\end{equation}

\begin{figure*}[t]
    \mbox{\includegraphics[width=0.3\textwidth]{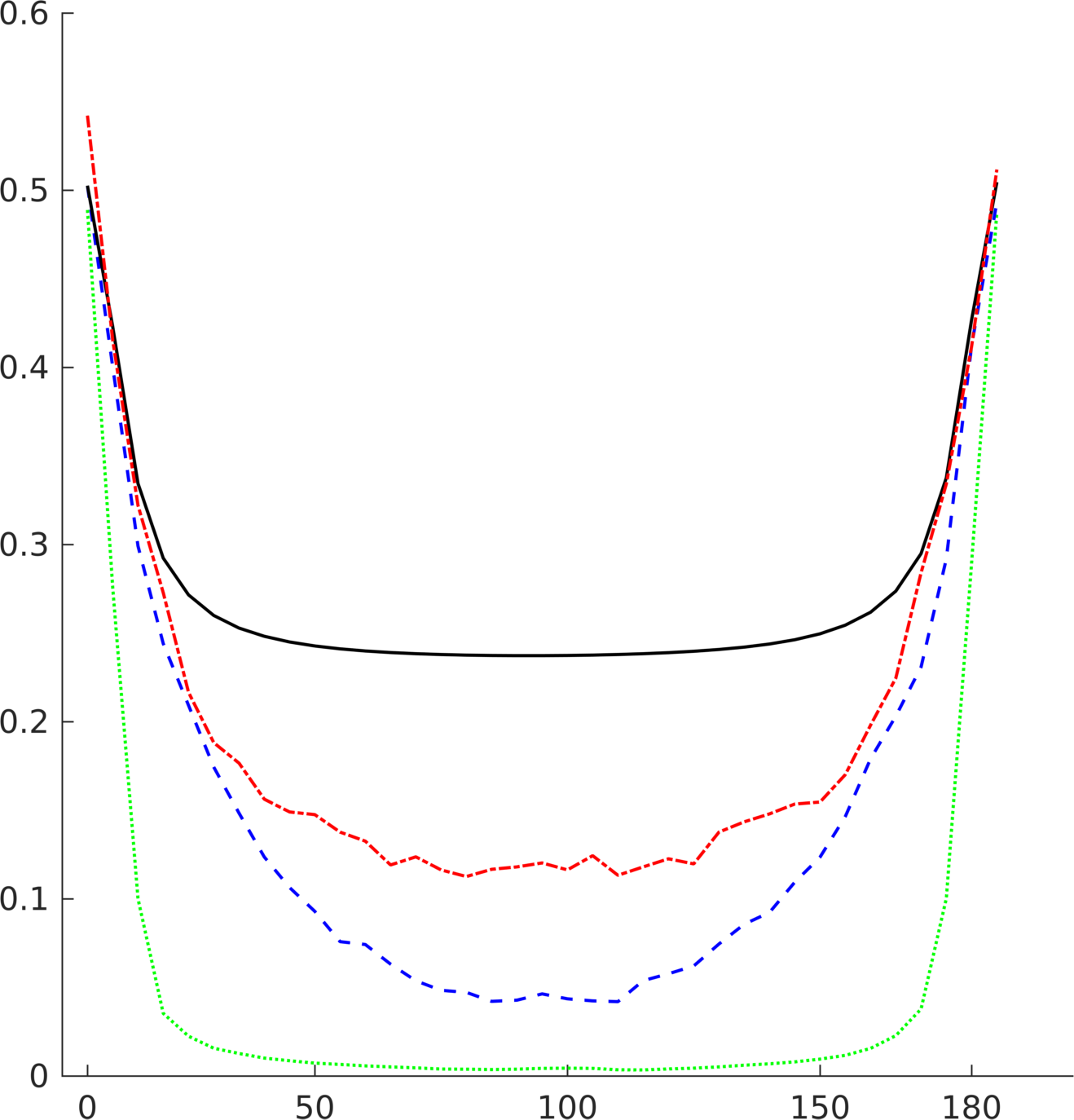}}
    \mbox{\includegraphics[width=0.3\textwidth]{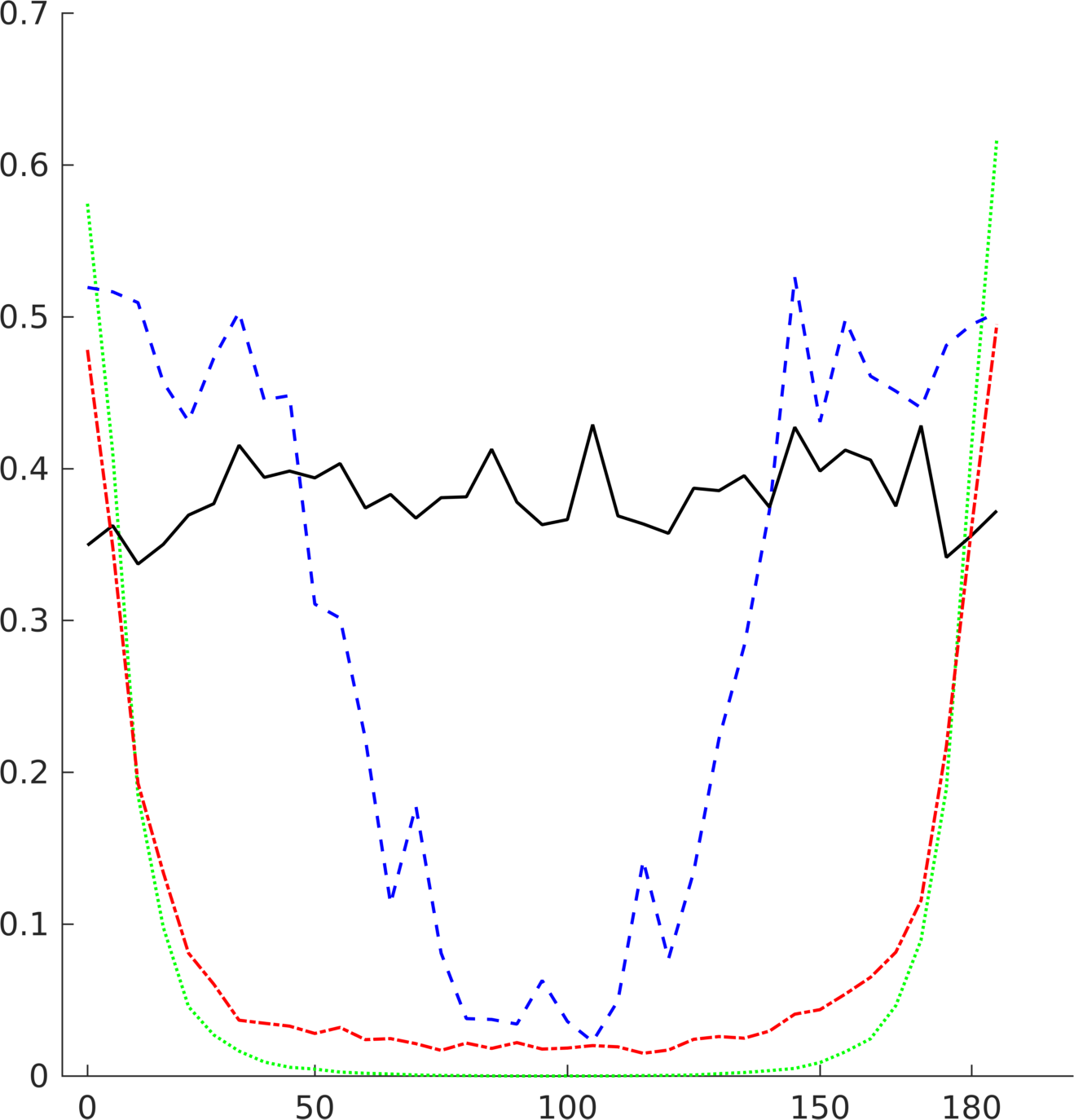}}
    \mbox{\includegraphics[width=0.3\textwidth]{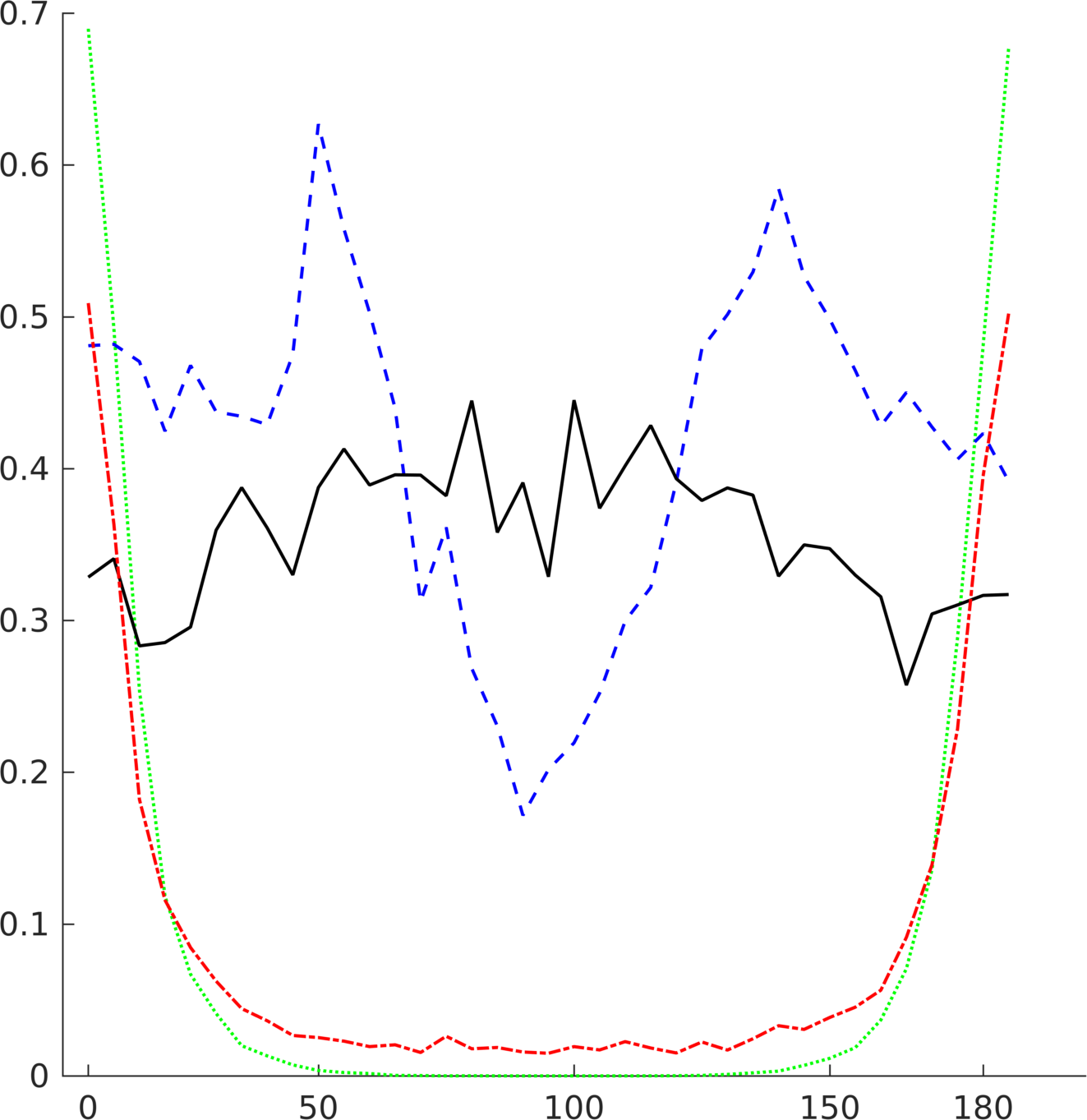}}
  \caption{Relative connectivities between the two clusters. The $x$-axes
    correspond to angles between the clusters (cf. the explanation in
    Sec.~\ref{Sec:Exps}), the $y$-axes to the relative connectivity $\xi$
    defined by Eq.~\eqref{Eq:xi}. Left: no noise;
    middle: noise with $\sigma = 1\%$; right: noise with $\sigma = 2\%$. The
    dotted green line corresponds to the lasso-method, the dash-dotted red
    line to the variational l0-norm
    optimization~\cite{Logsdon2012:L0Variational}, the dashed blue line to the
    proposed Selective Dantzig Selector, and the solid black line to the
    proposed Subspace Selector.}
\label{Fig:Res}
\end{figure*}

\subsection{Subspace Selector}

To further reduce the effect of spurious singular values, it is possible 
to fit an $d^{(l)} \approx \mbox{length}(S_{y})$ dimensional orthonormal 
basis $B$ to $X^\star$
\begin{equation}
B = \argmax \sum
 \|B X^\star_j\|_2^2, \quad j=1,\ldots,\mbox{length}(S^e_{y}).
\end{equation}

We can now choose additional points based on the Euclidean distance to 
the subspace $\mbox{span}(B)$ by 
\begin{equation}
  \argmax \limits_{j} \| P_B \Bx_j \|_2^2  > \delta, \quad j \in \{1,\ldots N\}
  \setminus \{\CalI(\By),\, S^e_{y} \}.
\end{equation}
for a threshold $\delta$. Here, $P_B$ is the orthogonal projector onto 
$\mbox{span}(B)$.


\section{Experiments}
\label{Sec:Exps}


To experimentally demonstrate the effect of disconnectivity between 
clusters of points on the same subspace, we created two normally 
distributed point clouds on the unit sphere  in $\RR^3$ with 
identical mean and standard deviation. Each clusters consists of $20$ 
points. One was then rotated in steps of $5^\circ$ from $0^\circ$ to
$180^\circ$. Afterwards, both point sets were multiplied by the same 
orthonormal $20 \times 3$ basis matrix. Each point was then normalized 
to length $1$.

For each angle between the clusters, normally distributed noise was 
added to the data. We used standard deviations of 
$\sigma = \{0,0.02,0.03\}$ which amounts to $0\%$, $2\%$, and $3\%$ 
noise. For each combination and noise magnitude, $10$ trials were 
performed, i.e. the data was perturbed $10$ times with different 
random noise. 

To measure the connectivity between the two point sets, we analyzed 
the affinity matrix $A = |C|+|C^\top|$
\begin{equation}
  A = 
  \begin{bmatrix}
    A_{11} & A_{12} \\
    A_{21} & A_{22} 
  \end{bmatrix}.
\end{equation}
The coefficients in the  $A_{11}$ and $A_{22}$ blocks indicate the 
connectivity within each cluster whereas the coefficients in 
$A_{21}=A_{12}^\top$ indicate the connectivity between the different 
clusters.

We can therefore measure the connectivity $\xi$ between the clusters by 
\begin{equation}
  \xi = \frac{1}{N} \sum \limits_{i=1}^{N} \frac{\sum \limits_{j=1}^{N_2}
    A_{12}(i,j)}{\sum \limits_{j=1}^{N} A(i,j)}
  \label{Eq:xi}
\end{equation}

The two algorithms proposed in Sec.~\ref{Sec:Algo} are compared against 
a lasso-estimator
\begin{equation}
\|\Bc_c\|_1 + \lambda \|Y_{-j}\Bc_c - \By\|_2
\label{Eq:lasso}
\end{equation}
which is optimized using CVX~\cite{CVX}. We used only Eq.~\eqref{Eq:lasso} 
because the results of Eq.~\eqref{Eq:L1Optim} are almost identical. Further, 
an algorithm to estimate the $l0$-norm using a variational 
approach~\cite{Logsdon2012:L0Variational} was included into the comparison. 
This is motivated by the idea that the $l0$-norm is not supposed to be 
effected by the bias that the $l1$-norm estimators have.

The results are shown in the plots in Fig.~\ref{Fig:Res}. The left plot 
corresponds to no noise, the middle one to a medium noise ($\sigma=1\%$), 
and the right plot to strong noise ($\sigma = 2\%$). The $x$-axes
correspond to angles between the clusters (cf. the explanation in
Sec.~\ref{Sec:Exps}), the $y$-axes to the relative connectivity $\xi$
defined by Eq.~\eqref{Eq:xi}. The dotted green line indicates the 
lasso-method~\eqref{Eq:lasso}, the dash-dotted red line the l0-norm 
optimization, the dashed blue line the proposed Selective Dantzig 
Selector, and the solid black line the proposed Subspace Selector.

As can be seen, the two proposed algorithms do not suffer as much from 
the gap between the two clusters as the other algorithms. Very surprisingly, 
the $l0$-norm estimator is also strongly affected.

Average affinity matrices computed by three different methods are shown in 
Fig.~\ref{Fig:CMats}. Apparently, for the lasso estimator the two clusters 
cause almost disconnected subgraphs at an angle of $45^\circ$ (left column 
in Fig.~\ref{Fig:CMats}).

\begin{figure*}[t]
    \subfigure{
      \begin{minipage}{0.045\textwidth}
      	{(a)}
      \end{minipage}
      \begin{minipage}{0.949\textwidth}
        \mbox{\includegraphics[width=0.3\textwidth]{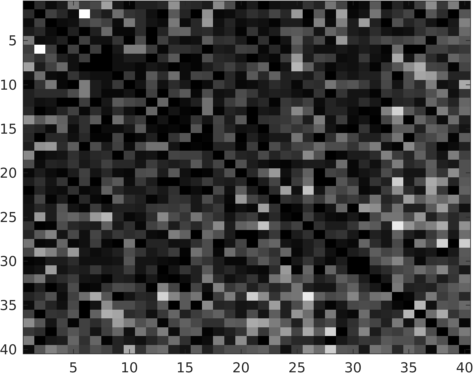}}
        \mbox{\includegraphics[width=0.3\textwidth]{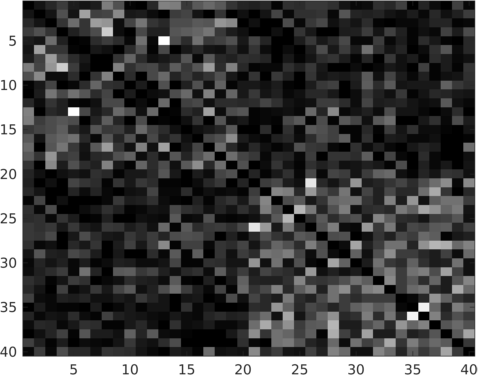}}
        \mbox{\includegraphics[width=0.3\textwidth]{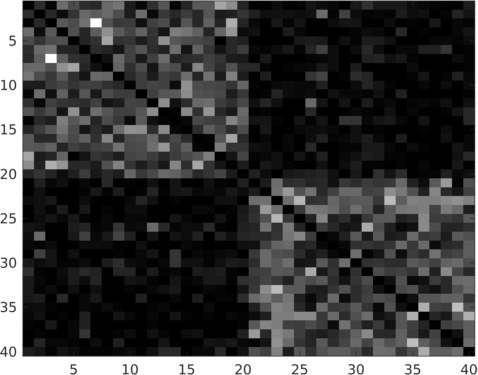}}
      \end{minipage}
    }
    \subfigure{
      \begin{minipage}{0.045\textwidth}
      	{(b)}
      \end{minipage}
      \begin{minipage}{0.949\textwidth}
        \mbox{\includegraphics[width=0.3\textwidth]{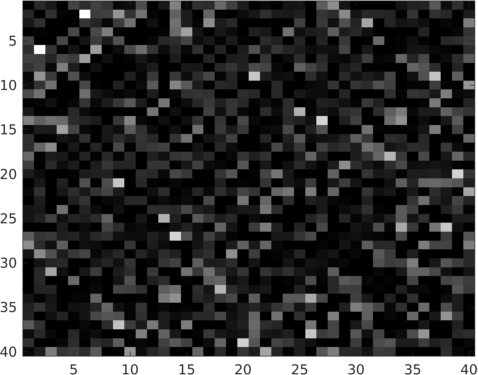}}
        \mbox{\includegraphics[width=0.3\textwidth]{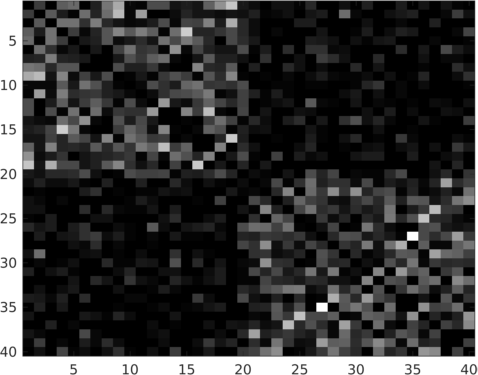}}
        \mbox{\includegraphics[width=0.3\textwidth]{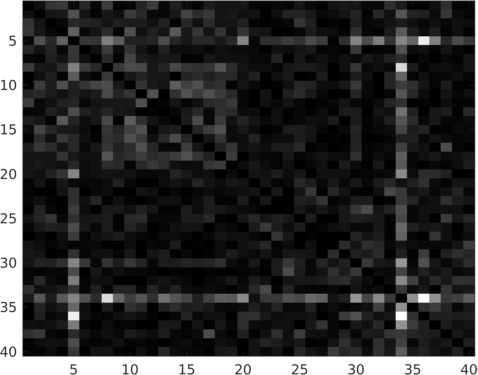}}
      \end{minipage}
    }
    \subfigure{
      \begin{minipage}{0.045\textwidth}
      	{(c)}
      \end{minipage}
      \begin{minipage}{0.949\textwidth}
        \mbox{\includegraphics[width=0.3\textwidth]{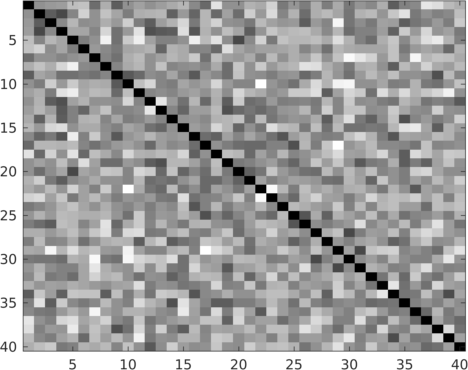}}
        \mbox{\includegraphics[width=0.3\textwidth]{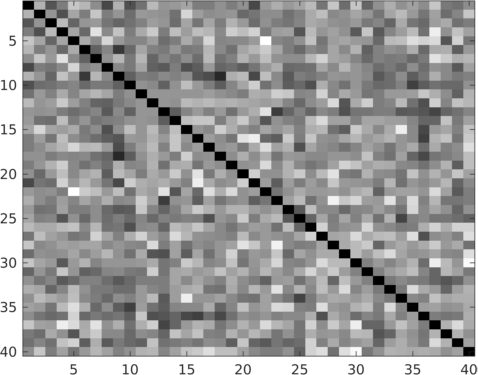}}
        \mbox{\includegraphics[width=0.3\textwidth]{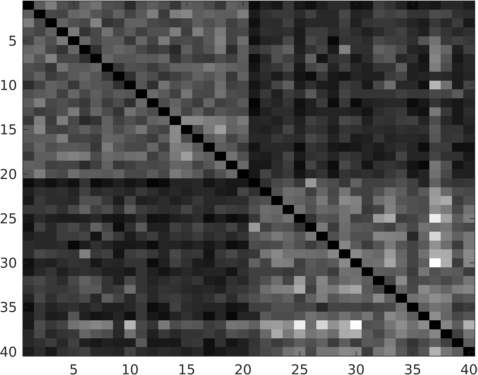}}
      \end{minipage}
    }
  \caption{Average affinity matrices $A=|C|+|C^\top|$ of the ten trials with
    their entries normalized to $[0,1]$. The left column are the average
    affinities if the angle between the clusters is $0^\circ$. The middle 
    column corresponds to angles of $10^\circ$, and the right one to 
    $45^\circ$. (a) average affinities computed by the lasso
    estimator~\eqref{Eq:lasso}. (b) affinities compute by the variational
    estimator of the $l0$-norm~\cite{Logsdon2012:L0Variational}. (c)
    affinities due to the proposed Selective Dantzig Selector.}
\label{Fig:CMats}
\end{figure*}





\section{Summary and Conclusions}
\label{Sec:Conclusions}

The topic of this work is sparse subspace clustering. The principle 
idea behind this class of algorithms is to approximate each data 
point by a linear combination of as few other data points as possible. 
This problem is often approached by solving $l1$-norm problems. 
The sparse coefficients are used to define edge weights  
of a graph. Computing a minimum cut through this graphs reveals which 
points are located on the same subspace.

Since there are few edges per vertex, it could be possible that vertices 
of a graph form disconnected subgraph although the corresponding points 
lie on the same subspace. In a previous work, the possibility for this to
happen was answered negatively if the subspace dimension is $2$ or $3$. 
Four $4$-dimensional subspaces, an example was given when disconnectivity 
occurs. 

This work investigates the relaxed definition of not exact disconnectivity 
but \emph{relative} disconnectivity. This means that all edge weights 
between two subgraphs are so low that the minimum cut either separates 
the two subgraphs, or -- if the number of clusters is fixed -- estimates 
a very erroneous result.

It is shown that this problem is caused by a gap between the points 
corresponding to the vertices of the subgraphs. In other words, if 
points on the same subspace fall into to separate clusters, then 
the two subgraphs necessarily have a so \emph{low} connectivity that 
a successful clustering is no longer possible.

This problem does not invalidate the result published 
in~\cite{Candes2012:SSC}, since the authors explicitly state that their 
Theorem~2.5 requires well distributed points. The idea in this work is to
consider the case that this assumption is not given.

Two different algorithms are proposed which are not as susceptible to 
the shown effect. Surprisingly, even an estimator of the $l0$-norm 
suffers from biased distributions of points.






{
\bibliography{aaai_2017}
\bibliographystyle{aaai}
}




\end{document}